\documentclass[lettersize,journal]{IEEEtran}
\def \etal {{\textit{et al}.\thinspace}}

\usepackage{amsmath,amsfonts}
\usepackage{algorithmic}
\usepackage{algorithm}
\usepackage{array}
\usepackage[caption=false,font=normalsize,labelfont=sf,textfont=sf]{subfig}
\usepackage{textcomp}
\usepackage{stfloats}
\usepackage{url}
\usepackage{verbatim}
\usepackage{graphicx}
\usepackage{cite}
\usepackage{amsthm}
\newtheorem{lemma}{Lemma}
\usepackage{booktabs}
\usepackage{multirow}
\usepackage{ulem}
\usepackage{hyperref}
\usepackage{xcolor}
\usepackage{amssymb}
\usepackage{cancel}
\newcommand\re[2]
{{\textcolor{black}{#2}}}

\newcommand\minre[1]
{{\textcolor{black}{#1}}}
\newcommand\vica[1]{{\textcolor{black}{#1}}}

\begin{document}

\title{Fast Wrong-way Cycling Detection in CCTV Videos: Sparse Sampling is All You Need}


\author{ Jing Xu$^{12\dag}$, Wentao Shi$^{23\dag}$, Sheng Ren$^{2}$, Lijuan Zhang$^{4}$, Weikai Yang$^{1*}$, Pan Gao$^{5*}$ and Jie Qin$^{5}$% <-this % stops a space
\thanks{This work was supported in part by the National Natural Science Foundation of China under Grants 62502413, 62272227.}
\thanks{$^{1}$ Data Science and Analytics Thrust, Information Hub, 
The Hong Kong University of Science and Technology (Guangzhou), Guangzhou, China.
        {\tt\small \{jing.xu@connect., weikaiyang@\}hkust-gz.edu.cn}}%

\thanks{$^{2}$ The College of Computer Science and Technology, Nanjing University of Aeronautics and Astronautics, Nanjing, China.
        {\tt\small \{jing.xu, shiwentao, rensheng\}@nuaa.edu.cn}}%
\thanks{$^{3}$ The School of Computer Science, Nanjing University, Nanjing, China.
        {\tt\small wentao.shi@smail.nju.edu.cn}}%
\thanks{$^{4}$ The College of Electronic and Information Engineering, Nanjing University of Aeronautics and Astronautics, Nanjing, China.
        {\tt\small lijuan.zhang@nuaa.edu.cn}}%
\thanks{$^{5}$ The College of Artificial Intelligence, Nanjing University of Aeronautics and Astronautics, Nanjing, China.
        {\tt\small \{pan.gao, jie.qin\}@nuaa.edu.cn}}%
\thanks{$\dag$Co-first author}%
\thanks{*Co-corresponding author}% <-this % stops a space
}

\markboth{Accepted by IEEE Transactions on Intelligent Transportation Systems in Dec. 2025}%
{Shell \MakeLowercase{\textit{et al.}}: A Sample Article Using IEEEtran.cls for IEEE Journals}

\IEEEpubid{0000--0000/00\$00.00~\copyright~2021 IEEE}


\maketitle

\begin{abstract}
Effective monitoring of unusual transportation behaviors, such as wrong-way cycling (i.e., riding a bicycle or e-bike against designated traffic flow), is crucial for optimizing law enforcement deployment and traffic planning.
However, accurately recording all wrong-way cycling events is both unnecessary and infeasible in resource-constrained environments, as it requires high-resolution cameras for evidence collection and event detection.
To address this challenge, we propose WWC-Predictor, a novel method for efficiently estimating the wrong-way cycling ratio, defined as the proportion of wrong-way cycling events relative to the total number of cycling movements over a given time period.
The core innovation of our method lies in accurately detecting wrong-way cycling events in sparsely sampled frames using a light-weight detector, then estimating the overall ratio using an autoregressive moving average model.
To evaluate the effectiveness of our method, we construct a benchmark dataset consisting of 35 minutes of video sequences with minute-level annotations.
Our method achieves an average error rate of a mere 1.475\% while consuming only 19.12\% GPU time required by conventional tracking methods, validating its effectiveness in estimating the wrong-way cycling ratio. Our source code is publicly available at: \href{https://github.com/VICA-Lab-HKUST-GZ/WWC-Predictor}{https://github.com/VICA-Lab-HKUST-GZ/WWC-Predictor}.
\end{abstract}

\begin{IEEEkeywords}
Wrong-way cycling, video analysis, tracking.
\end{IEEEkeywords}

\section{Introduction}
% - problem statement

% \re{It occurs when a cyclist travels against the established direction of travel for vehicles on a roadway or designated cycling path.}{}
% \re{This behavior poses significant safety risks for both cyclists and other road users, as it goes against established traffic rules and increases the likelihood of accidents and collisions.}
% While effective for their intended purpose, these systems typically require specialized cameras and significant computational resources for detailed tracking and analysis per video stream \cite{suttiponpisarn2022autonomous, shubho2021real, 10725591}.
% This resource intensity inherently restricts their widespread deployment for comprehensive urban monitoring. 
Wrong-way cycling refers to the behavior of riding a bicycle or e-bike in the opposite direction of the designated traffic flow. 
\vica{This behavior is a serious violation that significantly increases the risk of collisions, posing safety risks for both cyclists and other road users.}
% \minre{Empirical evidence suggests that the phenomenon is far from rare: for example, a U.S. study found that although only 2.7\% of total cycling distance was conducted wrong‑way, 42\% of bicycle trips included at least one wrong‑way segment \cite{DHAKAL2018145}. Further, observational work in Germany shows that wrong‑way cyclists are engaged in many critical interaction situations and sometimes serious crashes—highlighting the elevated danger of unexpected directionality in mixed traffic \cite{leschik2024interactions} .}
\vica{While robust enforcement systems exist for motor vehicle violations, they typically rely on high-resolution CCTV cameras and significant computational resources for license plate recognition and violation evidence collection \cite{CCTV-video, 10725591}.}
\minre{The high cost and resource intensity of such systems restrict their widespread deployment for non-motorized transport monitoring, making them impractical for continuous, real-time bicycle tracking.}
% \minre{The high cost and resource intensity of such systems restrict their widespread deployment for non-motorized transport monitoring, especially at city scale, making them impractical for continuous, real-time bicycle tracking or for estimating aggregate metrics like the wrong-way cycling ratio.}
\minre{Moreover, in typical urban CCTV views, bicycles and e-bikes appear much smaller and are more frequently occluded than motor vehicles, which makes multi-object tracking pipelines particularly prone to identity switches and orientation errors under low-resolution and low-frame-rate settings.}
\minre{As a result, tracking-based wrong-way cycling detection often requires high-resolution cameras and dense frame sampling to remain reliable.}

Given these constraints, our focus shifts from identifying and penalizing individual offenders to a macro-level safety management and transport system design.
Specifically, we focus on a key metric: the wrong-way cycling ratio, i.e., the proportion of wrong-way cycling instances relative to the total number of cycling movements over a given time period.
This ratio is crucial for assessing area-specific safety levels, identifying hotspots that require targeted interventions, and providing vital data for urban planning and safety optimization. 
% \minre{A straightforward way to reduce the computational burden is to analyze only short consecutive video clips at each location. However, because wrong-way cycling events are rare and typically occur intermittently, such short observation windows are prone to missing many wrong-way instances and thus systematically underestimate the long-term wrong-way cycling ratio.}
The primary need, therefore, is to develop an efficient and scalable method to calculate this ratio across entire road networks, leveraging the resource-constrained CCTV feeds already widely deployed for general traffic observation.

\IEEEpubidadjcol

To address this need, we propose the $\mathbf{W}$rong-$\mathbf{W}$ay $\mathbf{C}$ycling Predictor (WWC-Predictor), a lightweight method designed to efficiently estimate the wrong-way cycling ratio using significantly fewer frames and fewer computational resources. 
This is achieved through sparse sampling supported by a Two-Frame WWC Detector, which precisely extracts orientation-based counts from each pair of frames.
To mitigate orientation errors caused by detection in sparse sampling (e.g., occlusion ambiguities), we introduce an ensemble method to cross-validate the cycling orientations detected in each pair of frames.
Subsequently, a temporal WWC estimator applies an ARMA model to convert validated frame-level counts into a video-level wrong-way cycling ratio\minre{, rather than relying on continuous tracking or fixed short-duration video clips.}
The extreme lightweight design enables real-time inference on edge devices \minre{and makes the proposed pipeline practical for large-scale monitoring even with low-resolution CCTV feeds,} while maintaining robust estimation.

To evaluate our method, and to foster future research on this task, we construct a benchmark containing 405 annotated images for \re{}{non-motor vehicle} detection task, 1199 images for orientation estimation task, and 4 fully annotated CCTV videos (35 minutes in total) for end-to-end validation.
The evaluation result on our benchmark demonstrates that WWC-Predictor achieves accuracy comparable to conventional tracking-based methods while requiring 6-10 times fewer frames and 4-6 times fewer computational resources, confirming the effectiveness of our method.

\vica{In summary, the contributions of our work includes
\begin{itemize}
    \item a two-frame wrong-way cycling detector, which robustly detects orientations of non-motor vehicles in each sparsely sampled frame pair,  % sampling method
    \item a temporal wrong-way cycling estimator, which forecasts video-level wrong-way cycling ratios utilizing ARMA time-series model, and % estimation method
    \item a benchmark dedicated to the wrong-way cycling ratio estimation task.
\end{itemize}
}
% It should be noted that there is no data leakage between training and validation datasets. 

% - overall contribution

% In summary, this paper introduces WWC-Predictor, a novel algorithm designed to analyze CCTV footage with minimal frames, \re{aiming to predict the probability of wrong-way cycling incidents effectively}{ which estimates the wrong-way cycling ratio via sparse orientation-based sampling}. Additionally, we establish the first benchmark dedicated to this specific downstream task. The source code and dataset will be released.

\section{Related Work}
\subsection{Detection of Wrong-way Incidents}
Wrong-way driving represents a closely related area that has garnered significant attention. \minre{Current detection methods can be broadly categorized into GPS-based and surveillance-based approaches. Among these, GPS-based methods detect wrong-way cycling by exploiting location traces collected from smartphones or bicycle-mounted devices.} \re{}{Gu~\etal}\cite{gu2017bikemate} proposed BikeMate, a ubiquitous bicycling behavior monitoring system with smartphones. \re{}{Hayashi~\etal}\cite{hayashi2021vision} presented a mobile system that performed vision-based scene analysis to detect potentially dangerous cycling behavior including wrong-way cycling. \re{}{Dhakal~\etal}\cite{dhakal2018using} used data collected from a smartphone application to explain wrong-way riding behavior of cyclists on one-way segments to help better identify the demographic and network factors influencing the wrong-way riding decision making. \minre{However, these GPS-based methods assume that all cyclists carry and actively use smartphones with location services, which is unrealistic for large-scale, passive traffic monitoring, and may also raise privacy concerns when collecting fine-grained trajectory data. Consequently, these methods are not directly applicable to our scenario.} % , and we focus on surveillance-based methods using roadside cameras

\minre{On the surveillance side,} \re{}{current studies} \cite{suttiponpisarn2022autonomous, shubho2021real, 10421458, 9230463, 9648579, manasa2023enhanced} have developed frameworks for the detection of wrong-way driving in CCTV footage. These frameworks employ a variety of multi-object tracking (MOT) methods, including FastMOT \cite{yukai_yang_2020_4294717}, DeepSORT \cite{SORT_2}, Kalman filter \cite{saho2017kalman} and centroid tracking \cite{9230463}, to analyze vehicular movements. 
For example, Suttiponpisarn~\etal \cite{suttiponpisarn2022autonomous} first utilizes FastMOT to segment video footage into one-minute intervals. This segmentation allows for the precise identification of vehicle start and end points through tracking models, with subsequent post-processing to determine vehicle orientation.

By contrast, other strategies involve analyzing entire video sequences, employing detection-based methods followed by post-processing to ascertain the direction of travel. \re{}{Suttiponpisarn~\etal}\cite{9648579,suttiponpisarn2022enhanced} developed a system utilizing YOLOv4-tiny \cite{bochkovskiy2020yolov4} to detect objects and DeepSORT tracking to track vehicles, in which two main algorithms called Road Lane Boundary detection from CCTV algorithm (RLB-CCTV) and Majority-Based Correct Direction Detection algorithm (MBCDD) are designed to reduce computational time.\re{}{Choudhari~\etal}\cite{choudhari2023traffic} introduced a continuous tracking method for monitoring the direction of motorbikes, comparing it against the expected direction of the lane. If the tracked orientation of a motorbike is contrary to the lane direction for at least 80\% of the observed time, it is identified as wrong-way riding.

In essence, the prevailing strategies for detecting wrong-way incidents predominantly rely on tracking methods. \re{While effective, these methods are noted for their time-intensive nature, underscoring a critical area for potential efficiency improvements.}{While effective, these methods are computationally intensive \cite{9648579,10421458}, often requiring significant GPU resources \cite{suttiponpisarn2022autonomous, 10725591} and processing time \cite{9230463}, which becomes particularly challenging when analyzing extensive CCTV footage from multiple cameras. \minre{Furthermore, because they depend on long-term multi-object tracking and accurate lane or boundary modeling, their performance can degrade under occlusions, low illumination, or complex maneuvers, leading to delayed or missed detection of truly dangerous wrong-way events. These computational and robustness issues represent critical limitations in real-world deployments and underscore the need for more efficient approaches that can maintain acceptable accuracy while reducing resource requirements.}}
% While these pipelines can
% achieve high detection accuracy, they typically rely on dense,
% long-term tracking across video frames and often require
% carefully labeled data for tuning, which limits their scalability
% to city-wide deployments.

\begin{figure*}[tb]
  \centering
  \includegraphics[width=\textwidth]{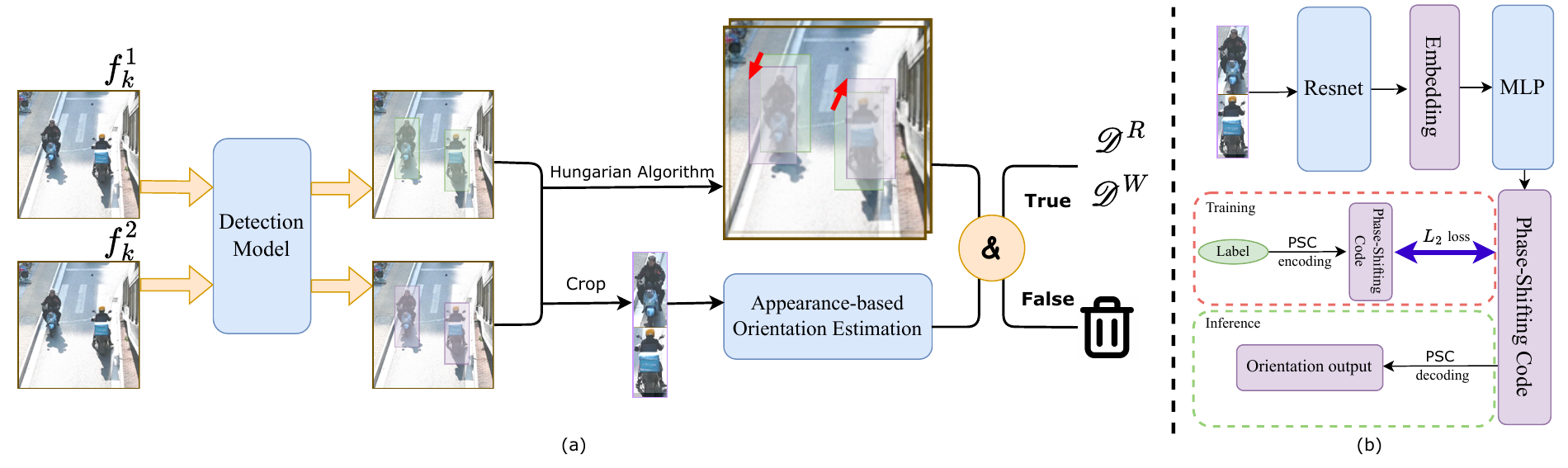}
  \caption{(a): Overview of Two-Frame Wrong-Way Cycling Detector, which consists of a detection model, an \re{}{appearance-based orientation} model and an And-strategy (shown as \(\&\) in the figure). \re{It takes two continuous frames as input, and output the amount of both right-way and wrong-way cycling.}{} (b): Training and inference pipeline of our proposed \re{}{appearance-based orientation estimation} model.}
  \label{Fig: ensemble}
\end{figure*}

\subsection{Orientation Detection}

Originally, detecting the orientation of non-motor vehicles in videos poses a dynamic challenge. However, within our sparse analysis framework, this issue transitions into a more static scenario, prompting us to delve into orientation detection within still images. Although direct analogues—tasks with an image input leading to an orientation output—are scarce, there exists a foundation of architectures addressing orientation in various contexts. Historically, orientation challenges have often manifested within the realm of oriented object detection \cite{Xie_Cheng_Wang_Yao_Han_2021, Han_Ding_Xue_Xia_2021, WEN2023119960}. This methodology aims not only to delineate the object with a bounding box but also to ascertain its orientation, thereby enhancing the precision of detection beyond the capabilities of conventional object detection techniques.

A novel contribution to this field is introduced by \re{}{Yu~\etal}\cite{Yu_2023_CVPR} through the development of a differentiable angle coder, termed the phase-shifting coder (PSC). This innovative approach tackles the issue of orientation cyclicity by translating the rotational periodicity across various cycles into the phase of differing frequencies, effectively addressing the rotation continuity problem. In our project, we incorporate the PSC to aid in resolving the challenges associated with orientation detection, leveraging its advanced capability to understand and quantify orientation within images accurately.

% \subsection{Model Ensemble}

% Ensemble learning \cite{Ganaie_Hu_Malik_Tanveer_Suganthan_2022} is a powerful technique that combines multiple individual models to achieve improved generalization performance. One approach to ensemble learning is multiple-model ensemble, where different models are fused together to enhance the performance of a specific task. The decision fusion strategy plays a crucial role in ensemble learning, as it determines how the outputs of the individual models are combined to achieve an effective ensemble.

% There are several popular ensemble strategies that have been widely used. These include unweighted model averaging \cite{Ju_Bibaut_van}, majority voting \cite{Majority_vote}, stacked generation \cite{stack_generation}, among others. Each strategy has its own strengths and weaknesses, and the choice of strategy depends primarily on the specific task and dataset. In our targeted task, we have devised a novel yet efficacious combination strategy, termed the AND-Strategy, which capitalizes on the synergistic potential of the individual models to amplify the overall performance.

\subsection{Mathematical Modeling of Traffic Flow}

Given our focus on sparse time sequences, it is essential to establish connections among these sequences rather than treating them as isolated samples. The modeling of traffic flow serves as a pertinent example of this approach.

\re{}{Tian~\etal}\cite{tian2021emd} suggested modeling the frequency components of network traffic using the ARMA model, illustrating a method to grasp the temporal dynamics in data. Similarly, \re{}{Peng~\etal}\cite{peng2021short} introduced an ARIMA-SVM combined prediction model to forecast urban short-term traffic flow, demonstrating the efficacy of integrating traditional statistical models with machine learning techniques for enhanced predictive accuracy. These methodologies underscore the significance of training models on a set of pre-measured network traffic data, enabling the ARMA or ARIMA-SVM models to capture the intrinsic characteristics of traffic flows and, subsequently, to forecast network traffic effectively.

% \re{}{In the aspect of multi-traffic flow relationship, graph signal processing\cite{Isufi_2017, 6494675, 10.1007/978-3-030-81638-4_3, 10146241} provides techniques for handling time-varying graph signals in a distributed manner, which can be applied to mutli-traffic flow modeling to capture spatial and temporal dynamics effectively.}
\re{}{Recent advances in Graph Signal Processing (GSP) have demonstrated significant potential for video analysis \cite{Giraldo2021Moving}. Fundamental to GSP is the extension of traditional signal processing concepts to irregular domains \cite{Shuman2013}, with ARMA graph filters \cite{Isufi_2017} providing a particularly relevant framework for our work. These techniques model signals evolving over graph structures, where temporal dynamics can be represented as graph signals along time-vertices \cite{Leus2023}, and ARMA graph filters capture recursive relationships in signal evolution \cite{Isufi_2017}. While our immediate application focuses on single-camera traffic flow, the GSP perspective offers pathways for future extension to multi-camera networks using graph-based fusion.}

In our analysis, although prediction is not our primary aim, we employ \re{similar techniques}{ ARMA model} to distill specific features that elucidate the relationships among samples. 

\section{\re{}{Method Overview}}
\vica{Taking a raw video \(\mathcal{V}\) as input, our method predicts the wrong-way cycling ratio through two integrated stages: the sparse detection and the temporal estimation.
In the sparse detection stage, we uniformly sample \(\mathcal{V}\) at fixed intervals \(T_{gap}\) to generate sequential frame pairs \(\mathcal{S} = \{ \mathcal{S}_0, \mathcal{S}_1, \ldots \}\), where each pair \(\mathcal{S}_k\) consists of two consecutive frames.
Each pair \(\mathcal{S}_k\) is processed by our \textbf{two-frame wrong-way cycling detector}, which produces sparse time-stamped counts of right-way cycling events \(\mathcal{D}^R_k\) and wrong-way cycling events \(\mathcal{D}^W_k\).
In the temporal estimation stage, the sparse detection results serve as input to our \textbf{temporal wrong-way cycling ratio estimator}, which models the number of events across the entire time period using an autoregressive moving average model to estimate the wrong-way cycling ratio.
This two-stage design enables efficient processing by combining targeted detection at key frames with statistical modeling for comprehensive temporal coverage.
% This two-stage design enables efficient video-scale analysis while maintaining robustness through multi-modal orientation verification.
% We account for vehicle persistence across sampling intervals through autoregressive estimation, which allows us to reconstruct the underlying arrival rates of right-way and wrong-way cyclists.
% The final WWC ratio is computed as the proportion of expected wrong-way cycling instances per sample relative to total cycling instances per sample over the video duration, transforming sparse measurements into comprehensive video-level predictions.
% This two-stage design enables efficient video-scale analysis while maintaining robustness through multi-modal orientation verification.
}

\re{}{\section{Two-Frame Wrong-Way Cycling Detector}}

% \re{}{Given a sampled two-frame pair \(\mathcal{S}_k = (f_k^1, f_k^2)\) from video \(\mathcal{V}\), Two-Frame Wrong-Way Cycling (WWC) Detector aims to determine vehicle orientations and classify right-way/wrong-way cycling instances through a three-stage pipeline, 1) detection, 2) orientation-aware model, and 3) ensemble strategy, as illustrated in Figure \ref{Fig: ensemble}. }
\vica{Given a two-frame pair \(\mathcal{S}_k = (f_k^1, f_k^2)\) from video \(\mathcal{V}\), our two-frame wrong-way cycling detector} determines vehicle orientations and classifies right-way/wrong-way cycling instances through a three-stage pipeline.
\vica{As illustrated in Figure \ref{Fig: ensemble}, our method integrates: 1) motion-based orientation estimation, which analyzes inter-frame object displacement to infer travel direction, 2) appearance-based orientation estimation, which leverages deep learning to predict orientation from vehicle appearance, and 3) ensemble validation, which cross-validates predictions from both methods to ensure robust classification.}

% which includes three modules: 1) \vica{detection-based orientation determination}, 2) orientation \vica{determination module for}, and 3) \vica{xxxx for xxx}.

\subsection{\vica{Motion-Based Orientation Estimation}}

\vica{The motion-based orientation estimation method determines vehicle travel direction by analyzing displacement between consecutive frames.
This process involves three steps: vehicle detection, vehicle matching, and orientation calculation.}

\vica{\textbf{Vehicle Detection}.
The first step detects non-motor vehicles in each frame pair \(\mathcal{S}_k\) through a detection model, which generates bounding box sets \(\mathbf{B}_1 = D(f_k^1)\) and \(\mathbf{B}_2 = D(f_k^2)\) for each pair.
Here, $D(\cdot)$ represents the detection function applied to frames, and we employ YOLOv5 \cite{YOLOV5} due to its proven industrial efficacy in real-time object recognition.}

\vica{\textbf{Vehicle Matching}.
To establish correspondences between detected vehicles across paired frames, we compute a similarity matrix using Intersection-over-Union (IoU) metrics.
The function \(F_{\text{IoU}}\) generates \(\mathbf{X}_{\text{IoU}} \in \mathbb{R}^{|\mathbf{B}_1| \times |\mathbf{B}_2|}\) where each element \(\mathbf{X}_{\text{IoU}}^{i,j} \) quantifies spatial overlap between detected boxes.
To exclude stationary objects, we apply a threshold mask:
\[
\mathbf{X}_{\text{IoU}} = \mathbf{X}_{\text{IoU}}\odot (\mathbf{X}_{\text{IoU}} < \text{IoU}_{\max}),
\]
where \(\text{IoU}_{\max}=0.98\) excludes high-overlap matches that likely represent stationary objects.
The optimal bipartite matching is then obtained using the Hungarian algorithm \(\mathbf{H}(\cdot)\), which produces matched index pairs \(L_{\text{match}} = \{(i_1,j_1),(i_2,j_2),\ldots\}\).}

\vica{\textbf{Orientation Calculation}.
For each valid match $(i,j)$ in \(L_{\text{match}}\), we compute the geometric orientation by analyzing the displacement vector between frame centroids:
\[
\begin{aligned}
O_{\text{det}} &= \text{arctan}(\mathbf{Cen}(\mathbf{B}_2^j) - \mathbf{Cen}(\mathbf{B}_1^i)),
\end{aligned}
\]
where \(\mathbf{Cen}(\cdot)\) denotes the centroid of a bounding box.}
 % \(O_\text{det}\) provides the primary orientation estimate for subsequent validation.

\begin{algorithm}
	\caption{Two-Frame Wrong-Way Cycling Detector}\label{algorithm}
	
	\begin{algorithmic}[1]
		\REQUIRE whole video $\mathcal{V}$, right-way orientation $O_{right}$  % 输入
        \STATE \re{}{Sparse sampling: $\mathcal{S} = \{ \mathcal{S}_0, \mathcal{S}_1, \ldots \}$ from $\mathcal{V}$, where each $\mathcal{S}_k = (f_k^1, f_k^2)$ is a pair of consecutive frames}
		\FOR{$(f_k^1, f_k^2)$ in $\mathcal{S}$}
		\STATE \re{}{Apply detection model:} $\mathbf{B}_1 \leftarrow D(f_k^1)$, $\mathbf{B}_2 \leftarrow D(f_k^2)$
		\STATE \re{}{Compute IoU:} $\mathbf{X}_{\textit{iou}} \leftarrow F_{\textit{iou}}(\mathbf{B}_1, \mathbf{B}_2)$
		\STATE $\mathbf{X}_{\textit{iou}} \leftarrow (\mathbf{X}_{\textit{iou}} < \textit{IoU}_{\textit{max}}) \cdot \mathbf{X}_{\textit{iou}}$
		\STATE \re{}{Apply Hungarian algorithm: }$L_{\textit{match}} \leftarrow \mathbf{H}(\mathbf{X}_{\textit{iou}})$
		\FOR{$i,j$ in $L_{\textit{match}}$}
            \STATE \re{}{Compute orientation in two methods:}
		\STATE $O_{\textit{det}} \leftarrow \mathbf{Orient}(\mathbf{Cen}(\mathbf{B}_1^i), \mathbf{Cen}(\mathbf{B}_2^j))$
		\STATE $O_{\textit{model}} \leftarrow \mathbf{Ave}(\mathbf{F}_o({\mathbf{B}_1^i, f_k^1}), \mathbf{F}_o({\mathbf{B}_2^j, f_k^2}))$
            \STATE \re{}{Pick out valid instances:}
		\IF{$\mathbf{AndStrategy}(O_{\textit{det}}, O_{\textit{model}})$}
		\STATE \re{Ans}{}$\re{}{OrientationList}.\mathbf{append}(\mathbf{Ave}(O_{\textit{det}}, O_{\textit{model}}))$
		\ENDIF
		\ENDFOR
		\ENDFOR
        \STATE \re{}{Count for final number:}
        \STATE $is\_Right \leftarrow \textbf{Dis}($\re{Ans}{}$\re{}{OrientationList}, O_{right})<$\re{120}{$\frac{2}{3}\pi$}
        \STATE $\mathcal{D}_R \leftarrow \mathbf{SUM}(is\_Right)$
        \STATE $\mathcal{D}_W \leftarrow \mathbf{SUM}(\neg is\_Right)$
  \RETURN $\mathcal{D}^R, \mathcal{D}^W$
	\end{algorithmic}
\end{algorithm}

\subsection{Appearance-Based Orientation Estimation}
\label{orientation}
\vica{While motion-based estimation provides reliable orientation cues from inter-frame displacement, it may fail in scenarios with insufficient movement, occlusion, or ambiguous trajectories.
To address these limitations, we introduce an appearance-based orientation estimation model that predicts vehicle direction directly from visual appearance.}

% \textbf{Model architecture}.
Figure \ref{Fig: ensemble} illustrates the architecture of our appearance-based orientation estimation model, which takes an image of vehicle as input, and output its facing orientation.
Since the model output a continuous and periodic variable, we applied Phase-Shifting Coder (PSC) \cite{Yu_2023_CVPR} to transform the discontinuous degree system into continuous $m$-dimension vector.
The PSC works as follows:

% \re{}{In the Two-Frame WWC Detector, an orientation-aware model is employed to further validate the orientation estimates obtained from the detection phase. This model} takes an image of vehicle as input, and output its facing orientation \re{. Its architecture can be seen in Figure \ref{Fig: ensemble}}{as shown in Figure \ref{Fig: ensemble} (b), aiming to precisely predict the vehicle’s direction }. Since the model aims to predict the objects' orientation, a continuous and periodic variable, we applied Phase-Shifting Coder (PSC) \cite{Yu_2023_CVPR} to transform the discontinuous degree system into continuous $m$-\re{dimention}{ dimension} vector. In our work, we assign $m$ as $3$. The PSC works as follows:

Encoding:
\begin{equation}
    x_i = \cos \left( \varphi + \frac{2i\pi}{m} \right), i = 1, 2, \ldots, m.
\end{equation}

Decoding:
\begin{equation}
    \varphi = -\arctan\frac{\sum^{m}_{i=1}x_i\sin(\frac{2i\pi}{m})}{\sum^{m}_{i=1}x_i\cos(\frac{2i\pi}{m})}.
\end{equation}
In our work, we set $m$ as $3$, which strikes a good balance between representation accuracy and computational efficiency.

More specifically, we apply pretrained backbone, ResNet-101 \cite{resnet} here, to generate embedding $\mathbf{b} \in \mathbb{R}^n$ for an image, and a linear layer is applied to convert the embedding $\mathbf{b}$ to vector $\hat{\mathbf{x}} \in \mathbb{R}^m$.

During training process, the label $ \varphi \in \left( -\pi, \pi \right] $ is encoded to $ \mathbf{x} \in \mathbb{R}^m $. Then the loss is computed as follows:

\begin{equation}
    l = \frac{1}{m}\|\hat{\mathbf{x}} - \mathbf{x}\|^2.
\end{equation}

During inference, the vector $\hat{\mathbf{x}} \in \mathbb{R}^m$ is decoded to $\hat{\varphi} \in \left( -\pi, \pi \right]$ as the final output.

In the orientation prediction task, we define the metric as the distance between the prediction and the label. Since it is a cyclic number, its formula can be expressed as:

\begin{equation}
\begin{aligned}
d = \max(\varphi , \hat\varphi) - \min(\varphi , \hat\varphi).
\\
\text{Error} = \left\{
\begin{array}{l}
d, \quad \text{if } d \leq \textcolor{blue}{\pi},\\
\textcolor{blue}{2\pi} - d, \quad \text{otherwise}.
\end{array}
\right. \\
\end{aligned}
\end{equation}

Here, $\varphi$ represents the label value and $\hat\varphi$ represents the predicted value. Both $\varphi$ and $\hat\varphi$ are in the range of $(0, $\re{360}{ $2\pi]$} and are measured in radians. The error is computed based on the difference between the maximum and minimum values of $\varphi$ and $\hat\varphi$. If this difference is less than or equal to \re{180}{ $\pi$}, the error is equal to $d$. Otherwise, if the difference is greater than \re{180}{ $\pi$}, the error is computed as \re{$360-d$}{ $2\pi - d$}. In the experimental part, we utilize this metric to evaluate the performance of our model. 
% About training policy of orientation-aware model, please see section \ref{sec:ablation}.
\begin{figure}[hbt]
        \centering  
        \includegraphics[width=\columnwidth]{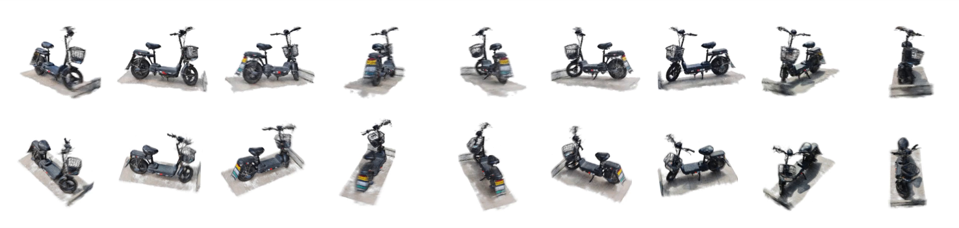} 
        \caption{One case reconstructed and generated by instant-ngp. }
        \label{nerf}
\end{figure}

To train our \re{}{appearance-based orientation estimation model}, \re{we utilize an effective pretrain-finetune architecture.}{ we initially plan to fine-tune a pretrained ViT \cite{dosovitskiy2021an} on limited manually collected and annotated orientation data. However, manually collected data often exhibits similar orientation patterns, resulting in imbalanced data distribution.} Recognizing the challenges associated with obtaining real-world data with precise orientation labels and addressing long-tail attributes, we generate synthetic data \re{. This synthetic data serves as a means to pretrain}{to conduct task-specific pretraining on} the model and to provide it with a beneficial bias towards orientation information. By leveraging this approach, we aim to enhance the model’s ability to understand and utilize orientation cues in real-world scenarios.

Firstly, we capture 360-degree videos using a regular camera in real-world settings. We then capture 30-40 frames from the video and utilize a trained COLMAP\cite{schoenberger2016mvs, schoenberger2016sfm} (Structure-from-Motion and Multi-View Stereo) system to estimate the camera's position and attitude parameters both inside and outside the captured scenes. Using these information, we reconstruct a 3D model using instant-ngp\cite{instant-ngp}.

Next, we leverage pre-designed camera poses to render images with corresponding orientation labels. To simulate real-world conditions, we render images from different heights for each orientation as shown in Figure \ref{nerf}. We capture images at every 10-degree interval with 3 different heights, resulting in 72 labeled images for each 3D model.

\re{}{This process allows us to generate a diverse dataset of labeled images, which serves as valuable task-specific pretraining data for our orientation-aware model.}
% In total, we created 12 3D models, resulting in 864 images used for pretraining. After the task-specific pretraining phase, the entire model is fine-tuned using real-world data.

% This process enables us to create a diverse dataset comprising labeled images, which serves as valuable pretraining data for our orientation-aware model. In total, we conducted 12 3D models, resulting in 864 images used for pretraining.
% After pretraining, the whole model is finetuned by real-world data. 

\re{}{\subsection{Ensemble Validation}}
\vica{Building on the orientation obtained from both motion-based and appearance-based methods, we employ an ensemble validation strategy to ensure robust and reliable orientation classification.}

% This validation process consists of three steps: 1) cross validation to verify prediction consistency between methods, 2) orientation fusion to combine validated predictions, and 3) binary classification to determine right-way versus wrong-way cycling.}

% \textbf{Cross validation}.
\vica{First, to obtain robust orientation estimates from the two complementary methods, we implement an And-strategy that integrates outputs from both the motion-based estimation ($O_{\text{det}}$) and the appearance-based model ($O_{\text{model}}$).
Specifically, predictions are considered valid only when the absolute difference between the two orientation estimates falls below a predefined threshold $|O_{\text{det}}-O_{\text{model}}|<\text{Div}_{\max}$.
We set $\text{Div}_{max}=\frac{2}{3}\pi$ in our implementation.}

\vica{For instances that pass the cross validation step, we compute the final orientation by averaging the two validated predictions.
This final orientation is then compared against the expected right-way direction $O_{\text{right}}$.
If the angular distance is less than $\frac{2}{3}\pi$, the instance is classified as a right-way cycling instance; otherwise, it is considered wrong-way cycling.
This process ultimately yields the counts of right-way cycling ($\mathcal{D}_R$) and wrong-way cycling ($\mathcal{D}_W$) instances.}

Mathematically, it can be shown that employing an And-strategy with an ensemble of two models can enhance overall performance compared to relying on a single model. 
Specifically, when the accuracy of both models exceeds 50\%, \re{}{this strategy guarantees improved performance.}
\begin{lemma}
    Assume that $P_1$ and $P_2$ represent the respective posterior probabilities that each of the two models errs, which are in the range of $(0, 0.5]$, $P_w = P_1P_2$ represents the possibility of two models erring simultaneously, and $P_{valid} = (1-P_1)(1-P_2) + P_1P_2$ represents the possibility of the sample being valid. The possibility of the ensemble model erring $P_3 = \frac{P_w}{P_{valid}} \leq min\{P_1, P_2\}$.
\end{lemma}

\begin{proof}
    Considering $P_3$ as a function related to $P_1, P_2$, we have 
    \begin{equation}
        P_3 = f(P_1, P_2) = \frac{P_1P_2}{(1-P_1)(1-P_2) + P_1P_2}.
    \end{equation}
    Deriving $f$ with respect to $P_1$ (or equivalently $P_2$) yields:
    \begin{equation}
        \begin{aligned}
        \frac{\partial f}{\partial P_1} &= \frac{-P_2^2 + P_2}{(1-P_1-P_2+2P_1P_2)^2}, \\
        \frac{\partial f}{\partial P_1} &> 0, \ \text{for}\ P_1, P_2 \in (0, 0.5].
    \end{aligned}
    \end{equation}
    
    This establishes $P_3$ as a monotonically increasing function relative to $P_1$ and $P_2$. As a result,
    \begin{equation}
        P_3 \leq f(P_1, 0.5) = P_1, P_3 \leq f(0.5, P_2) = P_2.
    \end{equation}
\end{proof}

% \subsection{Full-Time Wrong-Way Cycling Predictor}
% \re{}{\section{Full-Time Wrong-Way Cycling Predictor}}
\section{\vica{Temporal Wrong-Way Cycling Ratio Estimator}}

\vica{Using the Two-Frame Wrong-Way Cycling Detector, we can determine the orientations of non-motor vehicles at discrete time points.
However, reconstructing temporal information from such sparse data remains challenging, since each instance may appear at multiple time points.
Our Temporal WWC Estimator addresses this by estimating the number of instances ($N_R$ and $N_W$) passing through each time interval $T_{\text{gap}}$ using the counts of instances ($\mathcal{D}_R$ and $\mathcal{D}_W$).}

\vica{Formally, let $N_k$ be the number of vehicles passed the time interval \(T_{gap}\) between $t_{k-1}$ and $t_{k}$, we have the estimation $\hat{N} _k = \mathcal{D}_k - \varphi \mathcal{D}_{k-1}$, where the parameter $\varphi$ quantifies the probability of a vehicle from the previous frame persisting into the subsequent one.}
To effectively estimate the parameter \(\varphi\), we employ the Auto Regressive Moving Average (ARMA) model, which can be generally described as \re{following}{follows}:

\begin{equation}
    \mathcal{D}_k = c + \epsilon_k + \sum\limits_{i=1}^p \varphi_i  \mathcal{D}_{k-i} + \sum\limits_{j=1}^q \theta_j \epsilon_{k-j},
\end{equation}
\vica{where $\mathcal{D}_k$ represents vehicle counts at time interval $k$, $c$ is the constant term representing baseline traffic flow, $\epsilon_k \sim \mathcal{N}(0, \sigma^2)$ captures random fluctuations that are not explained by the past values, and $\varphi_i$ and $\theta_i$ are the auto-regressive and moving average coefficients, respectively.}

The sum \( \Sigma_{i=1}^p \varphi_i \mathcal{D}_{k-i} \) is the autoregression (AR) part, where \( p \) is the order of the AR process.
Each \( \varphi_i \) is a parameter that multiplies the number of vehicles at a previous time point \( k-i \), indicating how past values are weighted in the model. In our scenario, $p$ is selected as one, so that this term is simplifed as $\varphi \mathcal{D}_{k-i} $. 

The expression \( \Sigma_{i=1}^q \theta_j \epsilon_{k-j} \) represents the moving average (MA) part, wherein \( q \) denotes the order of the MA process.
Each coefficient \( \theta_j \) corresponds to a parameter that is applied to the lagged error term \( \epsilon_{k-j} \), illustrating the impact of historical forecast errors on the current value.
Within our traffic flow context, this encapsulates the momentum of vehicular movement, particularly in scenarios involving right-way cycling.
Conversely, for instances of wrong-way cycling, which are irregular and unforeseen, the MA component is disregarded, thereby setting \( q \) to zero.

\vica{After building our ARMA models, we can estimate the parameters $\phi$ and $\theta$ using maximum likelihood estimation~\cite{josef_perktold_2023_10378921, mcleod2008faster}.
We then derive sets of \(\hat{N}^R\) and \(\hat{N}^W\), which can be used to calculate the wrong-way cycling ratio:
\begin{equation}
    \text{WWC Ratio} = \frac{{E}(N_W)}{E(N_R)+E(N_W)}=\dfrac{\sum_{i}\hat{N}_{i,W}}{\sum_{i}\hat{N}_{i,R}+\sum_{i}\hat{N}_{i,W}}.
    \label{ratio}
\end{equation}}

\begin{figure*}[hbt]
        \centering  
        \includegraphics[width=\textwidth]{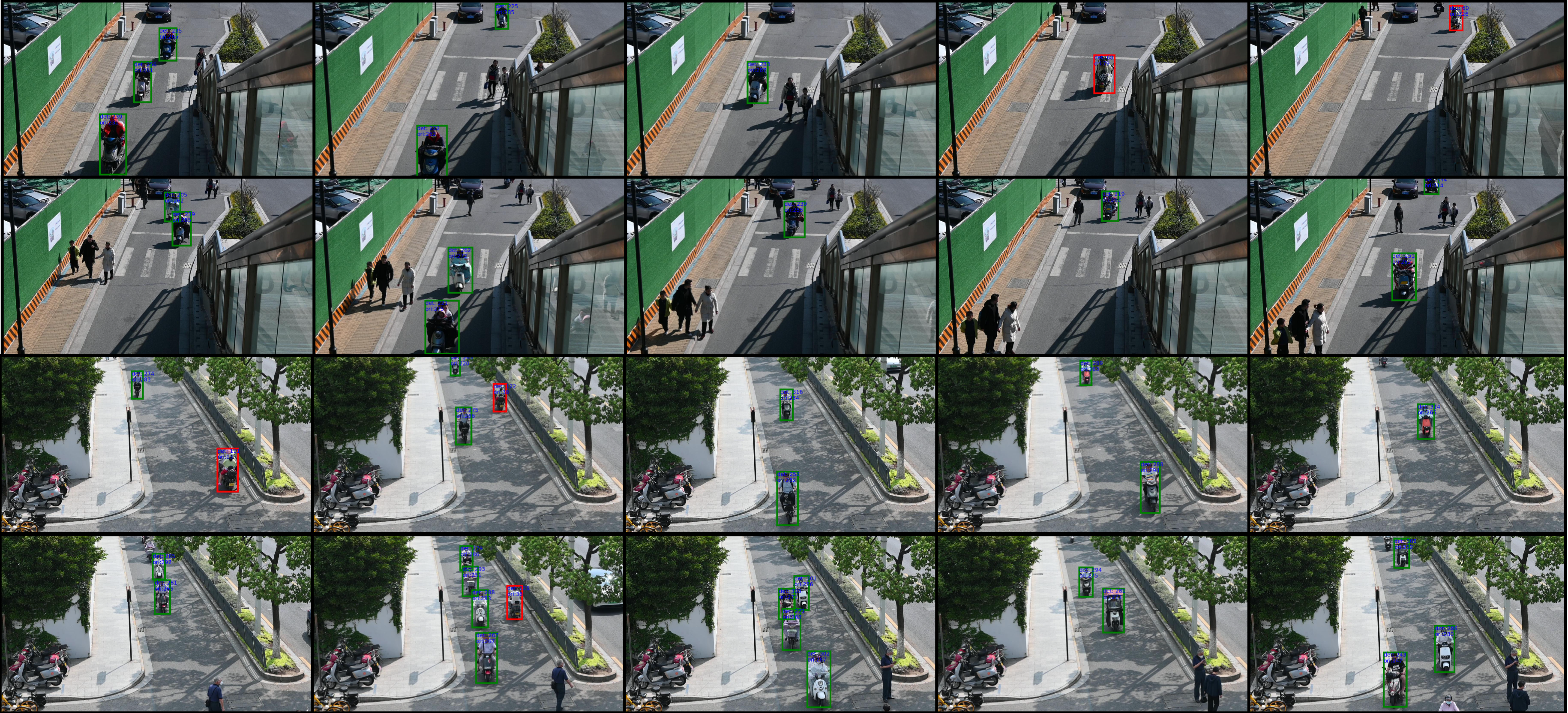} 
        \caption{Visualization from the validation videos for wrong-way cycling prediction in video, where green bounding box means right-way cycling, red means wrong-way cycling. `Det' and `Ori' \re{}{refer to the} output orientation from \re{}{the detection-based branch and the appearance-based branch.}}
        \label{val_data}
\end{figure*}

\section{Experiments}

\subsection{Benchmark Datasets}
\re{To support training and evaluation of our method, and to foster future research on this task, we propose three distinct datasets: }{Given that existing datasets related to non-motor vehicles are either not publicly available or exhibit varying quality, we recognized the need to create a new dataset from scratch. This effort aims to enhance the reliability of our research and support the training and evaluation of our method. To facilitate further research in this area, we propose three distinct datasets:} the non-motor vehicle detection dataset, the orientation prediction dataset, and the WWC ratio estimation dataset.
Each dataset serves a specific role in our model development and evaluation process.

\subsubsection{\re{}{Non-motor Vehicle Detection dataset}}
\re{}{
This dataset is designed for training the detection model.
It contains 223 images with 474 non-motor vehicle bounding boxes captured under diverse conditions.
The training-validation split is 8:2, which our experiments confirm as sufficient for effective model training.}

\subsubsection{\re{}{Orientation Prediction Dataset}}
This dataset is designed for training the appearance-based orientation prediction model.
It comprises three subsets: the pretraining set, the finetuning set, and the validation set.
The pretraining set consists of synthetic images generated using instant-ngp.
It includes 12 distinct 3D models, and for each model, we generate 72 labeled images captured from two different heights.
The details have been introduced in Sec. \ref{orientation}.
% In total, the pretraining set contains 864 images.
The finetuning set and the validation set contain 1,060 and 169 real-world images with manually-labeled orientation, respectively.

\begin{figure}[hbt]
        \centering  
        \includegraphics[width=0.9\columnwidth]{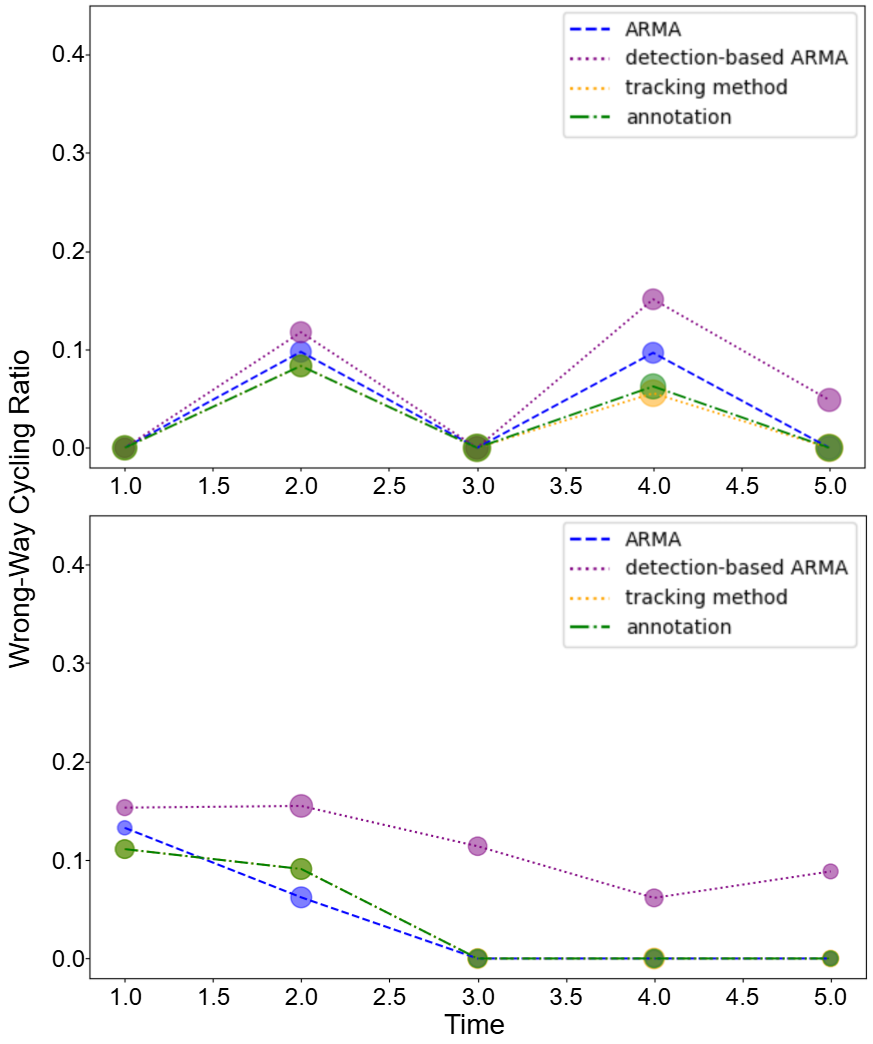} 
        \caption{Minute-level comparison between WWC-Predictor, detection-only WWC-Predictor, and tracking-based method (SORT). The size of the circle showcases its scale of numbers. }
        \label{exp1}
\end{figure}

\subsubsection{\re{}{WWC Ratio Estimation Dataset}}

% To establish a robust evaluation metric, we collect
\re{}{This dataset comprises four videos captured by road cameras situated in various locations: three short videos, each lasting 5 minutes (denoted as Case 1, Case 2, and Case 3), and one long video lasting 20 minutes (denoted as Case 4). As illustrated in Figure \ref{val_data}, which displays screenshots and detection results, this diverse collection of footage lays a solid foundation for evaluating the performance of different methods in terms of both speed and accuracy. By including both short and long videos, we can assess the methods' ability to predict the wrong-way cycling ratio under various conditions, thereby validating their effectiveness in capturing both short-term behaviors and longer-duration patterns.}
% These videos form the foundation for assessing the performance of different methods in terms of both speed and accuracy.
% By incorporating a diverse array of video footage, as depicted in Figure \ref{val_data}, we ensure a comprehensive evaluation of the methods across different real-world conditions.
Furthermore, we annotate instances of wrong-way and right-way cycling every minute, enabling us to assess performance on a minute-by-minute basis.
This evaluation strategy allows us to determine the effectiveness and reliability of various methods in accurately and efficiently predicting the wrong-way cycling ratio in road camera videos. 

\subsection{Baselines}

\re{}{
To comprehensively assess the performance of our proposed WWC-Predictor, we establish a baseline comparison that includes both external tracking methods and a simplified version of our own method.
This simplified method isolates the detection component, omitting the ensemble strategy from the full WWC-Predictor. In this detection-only baseline, we compute the orientation ($O_{det}$) between the center points of the two bounding boxes as the final output, providing a clear comparison to the more complex ensemble approach.}

\re{}{For the tracking component, since we lack annotated tracking data for training end-to-end trackers, we use tracking-by-detection methods. We employ SORT \cite{SORT}, DeepSORT \cite{SORT_2}, and ByteTrack \cite{zhang2022bytetrack}, with ByteTrack representing the state-of-the-art in detection-based tracking. }

\subsection{Implementation Details}
\re{}{We employ the YOLOv5-m architecture for the detection module, which strikes a balance between accuracy and computational efficiency, making it suitable for real-time applications. To predict orientation, we utilize a refined Resnet-101 network combined with a Phase-Shifting Coder (PSC) to capture fine-grained spatial features and temporal shifts in the cycling trajectory. Also, we use the same Resnet-101 as a feature extractor in the DeepSORT baseline. We configure the system to sample video frames every 2 (or 4) seconds, ensuring a consistent temporal resolution for accurate tracking.} \re{Table \ref{main_exp} presents a comparison of our method against the baselines on our benchmark dataset.}{}GPU processing time was measured on a single RTX 3080Ti, with serial inference for the detection model and parallel inference for the appearance-based orientation estimation model (in cases processing a single image).

\subsection{Results}

\begin{table*}[ht]
\centering
\caption{Comparison of our method against others on our benchmark. ``Error" refers to the absolute deviation between the predicted ratio and the ground truth, while ``Overall Error" represents the mean of these four discrepancies.  \re{}{\(T_{gap}\)} denotes the temporal interval between successive samples in this methodology. ``Time" refers to GPU time consumed for processing video of one minute. \re{}{Case 1-3 refer to three 5-minute videos, while case 4 refers to a 20-minute-long video.}}
\begin{tabular}{cccccccccccc}
\toprule
& \multirow{2}{*}{ \re{}{\(T_{gap}\)}} & \multicolumn{2}{c}{\re{Short 1}{ Case 1}} & \multicolumn{2}{c}{\re{Short 2}{ Case 2}} & \multicolumn{2}{c}{\re{Short 3}{ Case 3}} & \multicolumn{2}{c}{\re{Long}{ Case 4}} & Overall  & \multirow{2}{*}{Time}  \\
& & Ratio                         & Error                        & Ratio                        & Error                         & Ratio                         & Error                        & Ratio                       & Error                       & Error                     &                                             \\ \midrule

SORT \re{}{\cite{SORT}}                     & 0.17s                   & 4.2\%                         & -0.3\%                       & 2.4\%                        & -0.1\%                        & 4.1\%                         & -0.2\%                       & 12.2\%                      & +1.0\%                      & 0.375\%                                  & 20.18s                                      \\
DeepSORT \re{}{\cite{SORT_2}}              & 0.17s                   & 3.8\%                         & -0.7\%                       & 2.4\%                        & -0.1\%                        & 4.1\%                         & -0.2\%                       & 11.5\%                      & +0.3\%                      & 0.325\%                                  & 29.37s                                   \\
ByteTrack \re{}{\cite{zhang2022bytetrack}} & 0.17s   & 4.3\%                         & -0.2\%                       & 2.5\%                        & +0.0\%                        & 4.1\%                         & -0.2\%                       & 12.1\%                      & +0.9\%                      & 0.325\%                                  & 24.11s                                     \\
ByteTrack \re{}{\cite{zhang2022bytetrack}} & 0.5s     & 3.0\%                         & -1.5\%                       & 7.7\%                        & +5.2\%                        & 1.6\%                         & -2.7\%                       & 14.3\%                      & +3.1\%                      & 3.12\%                                   & 15.85s              \\
\midrule
\re{}{Detection only WWC Predictor (Ours)}          & 2s                             & 11.9\%                        & +7.4\%                       & 6.2\%                        & +3.7\%                        & 11.8\%                        & +7.5\%                       & 17.3\%                      & +6.1\%                      & 6.20\%                                     & 3.83s                

\\

\re{}{WWC Predictor (Ours)}                 & 2s                    & 3.9\%                         & -0.6\%                       & 3.6\%                        & +1.1\%                        & 4.9\%                         & +0.6\%                       & 14.8\%                      & +3.6\%                      & 1.475\%                                  & 4.50s                 
\\
\re{}{WWC Predictor (Ours)}   & 4s 
& 2.7\%                         & -1.8\%                       & 2.1\%                        & -0.4\%                        & 3.8\%                         & -0.5\%                       & 16.6\%                      & +5.4\%                      & 2.025\%                                      & 2.95s               
\\
\bottomrule
\vspace{2mm}

\end{tabular}

\label{main_exp}
\end{table*}

\begin{figure}[hbt]
        \centering  
        \includegraphics[width=0.9\columnwidth]{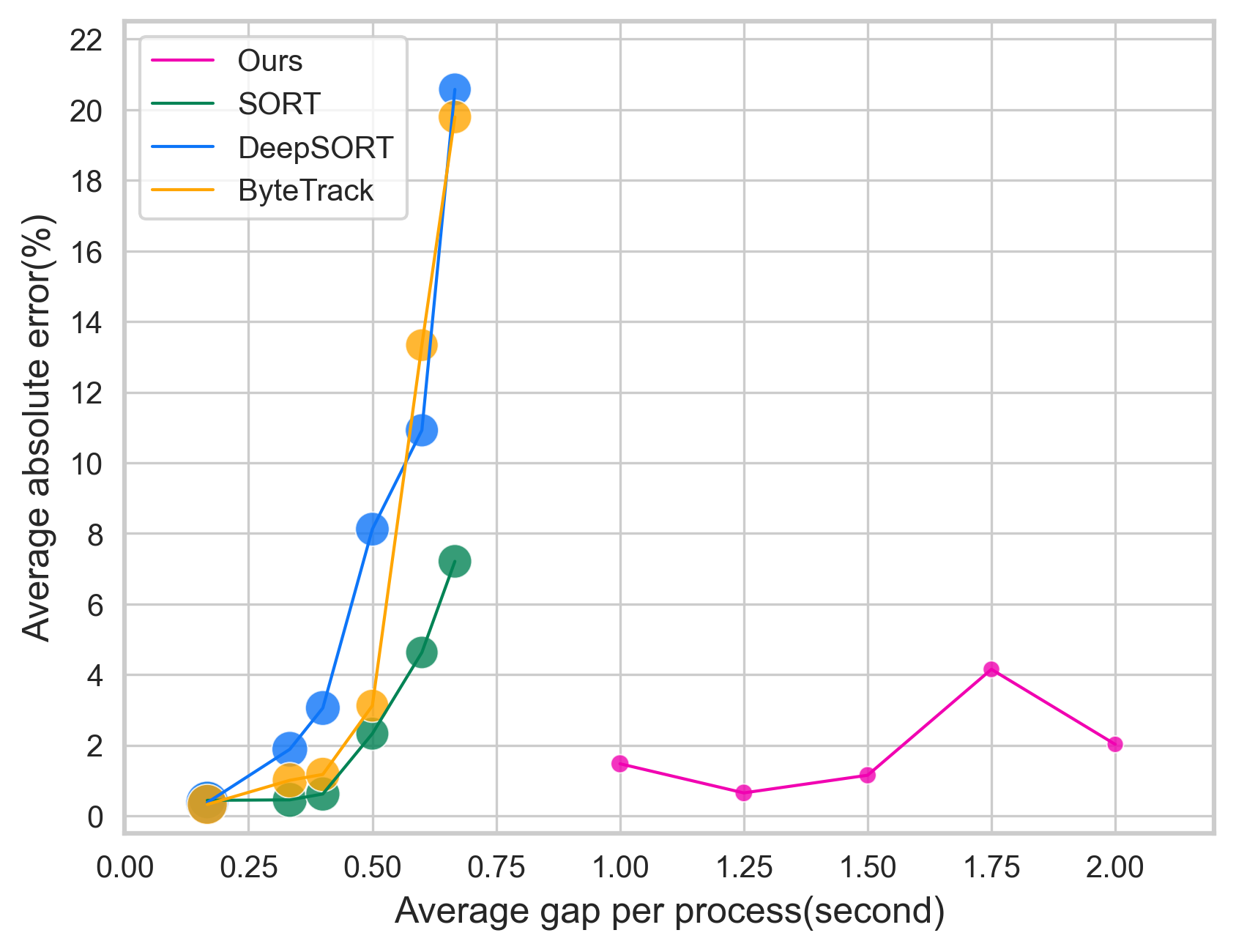} 
        \caption{Processed information scale, algorithm speed and performance comparison between WWC-Predictor, SORT, DeepSORT and ByteTrack. The size of circle showcases algorithm speed. }
        \label{fig:time}
\end{figure}

\re{Overall}{As shown in Table \ref{main_exp}}, the comprehensive WWC-Predictor method exhibits a competitive absolute error rate of 1.475\% and a swift inference time of 4.50 seconds per video minute. The detection-only variant of WWC-Predictor demonstrates a higher error margin due to the inherent uncertainty of the detection model, which leads to a significant proportion of True-Negative errors. Conversely, traditional tracking methods showcase a lower error rate owing to their robust frame-to-frame relations and the simplicity of time-dimensional prediction. % at the expense of substantially higher computational resource requirements.
\re{}{However, these methods require substantially higher computational resources.}

Figure \ref{fig:time} presents a comparative analysis of the WWC Predictor against SORT and DeepSORT across varying time interval scales. It is evident that both SORT and DeepSORT require relatively short time intervals to maintain effective tracking, whereas our WWC Predictor demonstrates robust performance over a specific range of larger intervals. 

We annotate instances of wrong-way and right-way cycling every minute, enabling us to assess performance on a minute-by-minute basis. This evaluation strategy allows us to determine the effectiveness and reliability of various methods in accurately and efficiently predicting the wrong-way cycling ratio in road camera videos. Figure \ref{exp1} presents a minute-level comparison between these three methods\re{}{, which shows that our method provides a more robust estimation than detection-only model and boosts the performance closer to traditional tracking methods.}

\subsection{Ablation Study}
\label{sec:ablation}

\begin{figure}[hbt]
        \centering  
        \includegraphics[width=0.9\columnwidth]{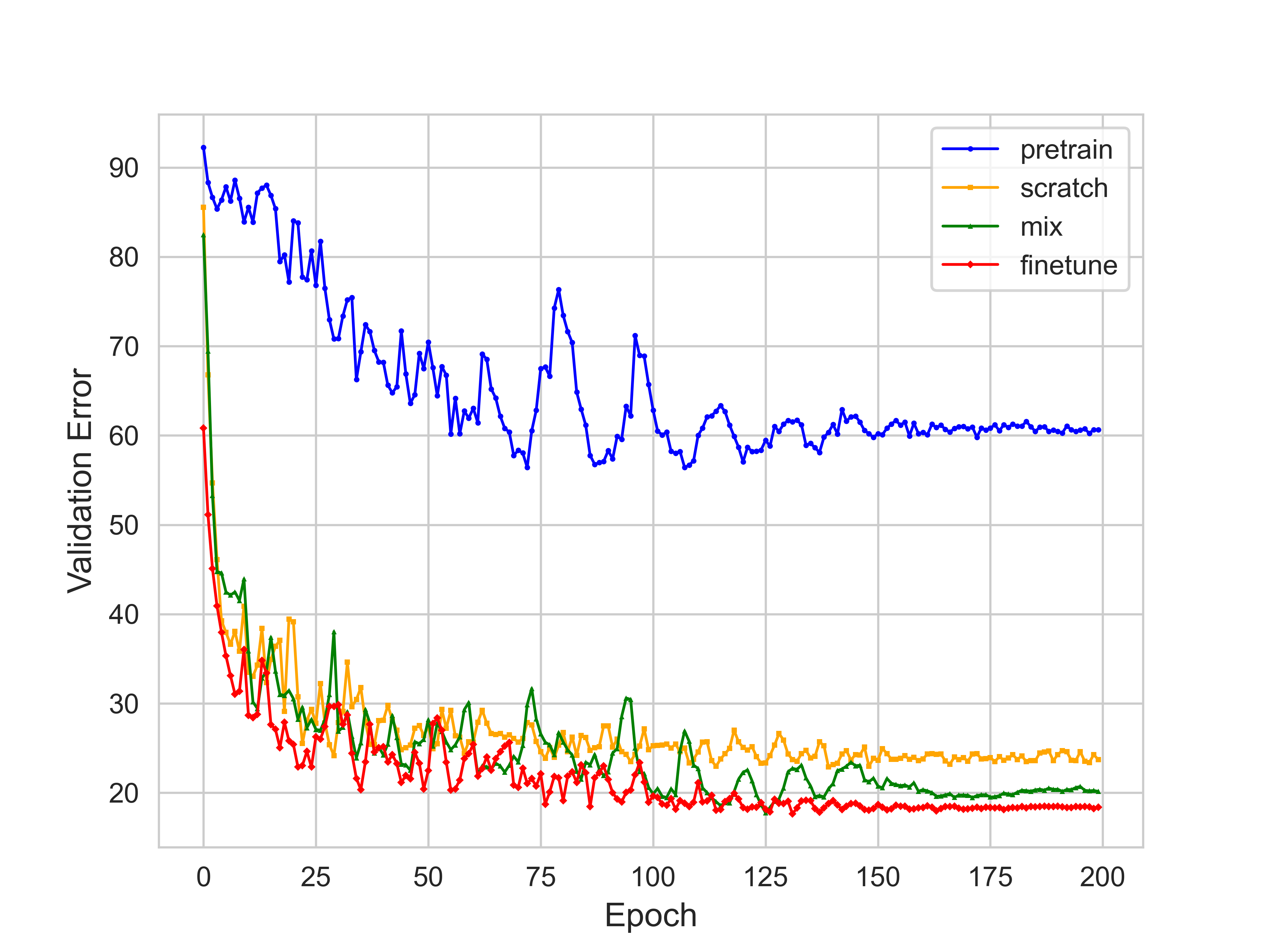} 
        \caption{Comparison of different training strategies for \re{}{appearance-based orientation estimation} model.}
        \label{fig:ablation}
\end{figure}

\subsubsection{Ablation on \re{}{appearance-based orientation estimation} model}
As described in Section \ref{orientation}, we have devised a pretrain-finetune architecture for training an \re{}{appearance-based orientation estimation} model. To evaluate the impact of the pretraining dataset and this specific training architecture, we conducted an ablation study. This study allowed us to assess the effectiveness and significance of these components in our model's performance.

Comparison of the validation error changes resulting from different training methods is presented in Figure \ref{fig:ablation}. The training methods are categorized as follows: 

\begin{itemize}
    \item ``Pretrain" denotes training exclusively with synthetic data.
    \item ``Scratch" refers to training solely with real-world data.
    \item ``Mix" represents training by combining both real-world and synthetic data.
    \item ``Finetune" indicates the process of fine-tuning on real-world data following a pretraining phase with synthetic data.
\end{itemize}

In our pretrain-finetune architecture, we successfully attain a minimal validation error of 17.63. Moreover, the performance post-convergence significantly surpasses that of training from scratch or employing a mixed dataset. While this error rate might initially appear substantial, it is deemed acceptable for its intended application as a verification model.

\subsection{Experiments of ARMA part}

\begin{figure*}
    \centering
    \includegraphics[width=\textwidth]{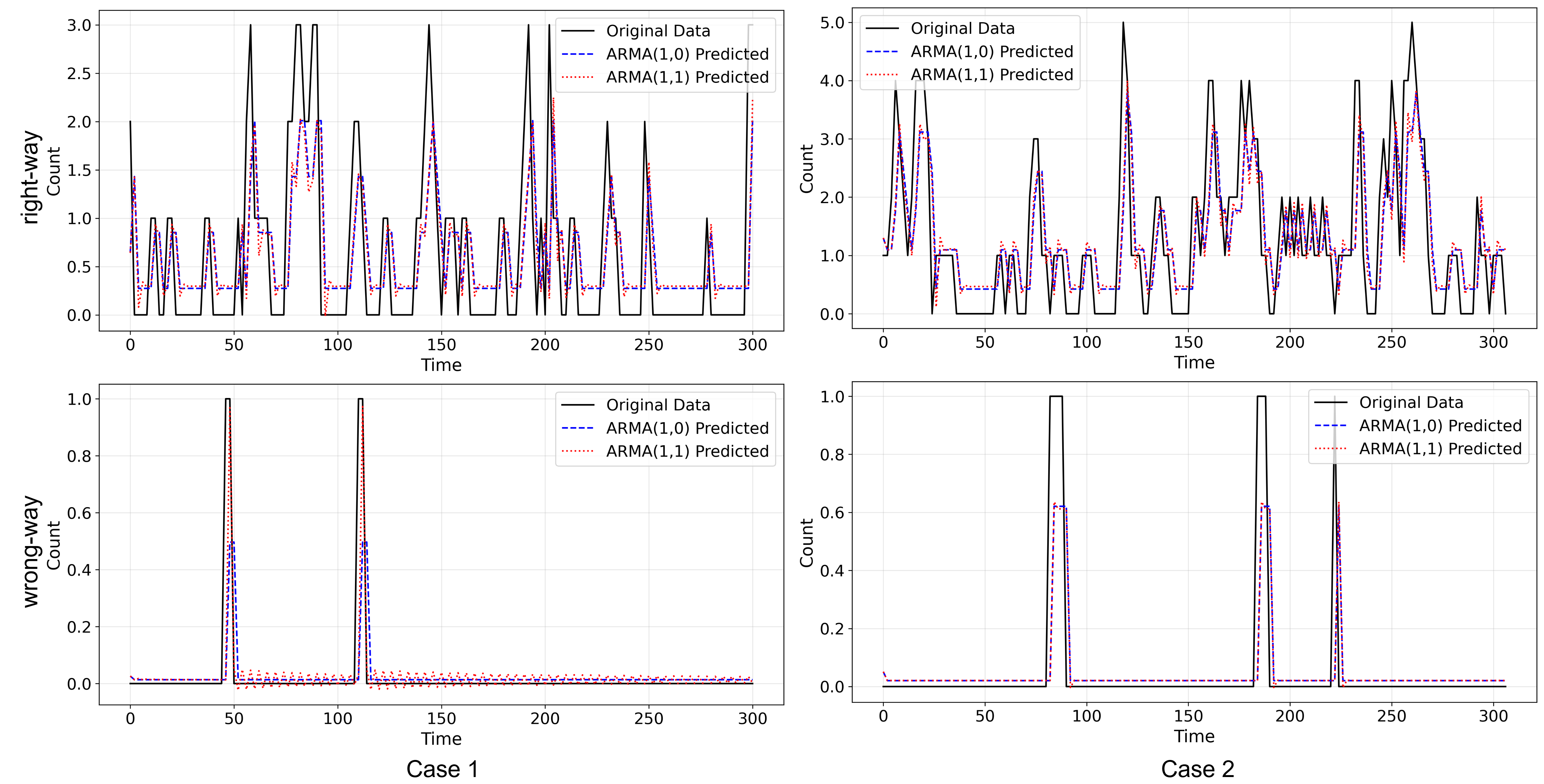}
    \caption{\re{In (a), we use ARMA(1,1) model to fit right-way cycling cases, while ARMA(1,0) for wrong-way cycling cases. In (b), use ARMA(1,1) model to fit both wrong-way cycling cases and right-way cycling cases.}{ Comparison between ARMA(1,1) and ARMA(1,0) for wrong-way and right-way data in case 1 and case 2.}  }
    \label{fig:ARMA}
\end{figure*}

\begin{table*}[htb]
    \centering
    \caption{Statistic results of auto regression(AR) part and moving average(MA) part for right-way cycling and wrong-way cycling. \re{}{Case 1-3 refer to three 5-minute videos, while case 4 refers to a 20-minute-long video.} }
\begin{tabular}{lllllllllllll}
\toprule
        & \multicolumn{6}{c}{ARMA model for Right-way cycling}                                                                                                                                        & \multicolumn{6}{c}{ARMA model for Wrong-way cycling}                                                                                                                                        \\
        & \multicolumn{3}{c}{AR}                                                                       & \multicolumn{3}{c}{MA}                                                                       & \multicolumn{3}{c}{AR}                                                                       & \multicolumn{3}{c}{MA}                                                                       \\
        & \multicolumn{1}{c}{std err} & \multicolumn{1}{c}{z} & \multicolumn{1}{c}{P\textgreater{}$\vert z\vert$} & \multicolumn{1}{c}{std err} & \multicolumn{1}{c}{z} & \multicolumn{1}{c}{P\textgreater{}$\vert z\vert$} & \multicolumn{1}{c}{std err} & \multicolumn{1}{c}{z} & \multicolumn{1}{c}{P\textgreater{}$\vert z\vert$} & \multicolumn{1}{c}{std err} & \multicolumn{1}{c}{z} & \multicolumn{1}{c}{P\textgreater{}$\vert z\vert$} \\ \midrule

\re{}{Case 1} & 0.093                       & 6.143                 & 0.000                                  & 0.107                       & 1.832                 & 0.067                                  & 0.087                       & 6.583                 & 0.000                                  & 0.080                       & 0.504                 & 0.614                                  \\
\re{}{Case 2} & 0.130                       & 3.428                 & 0.001                                  & 1.320                       & 1.482                 & 0.138                                  & 0.316                       & -0.015                & 0.988                                  & 0.282                       & 3.512                 & 0.000                                  \\
\re{}{Case 3} & 0.095                       & 5.857                 & 0.000                                  & 0.115                       & 1.603                 & 0.109                                  & 0.491                       & -0.146                & 0.884                                  & 0.470                       & 0.751                 & 0.453                                  \\
\re{}{Case 4}    & 0.048                       & 10.783                & 0.000                                  & 0.055                       & 3.800                 & 0.000                                  & 0.071                       & 6.094                 & 0.000                                  & 0.083                       & -1.021                & 0.307                                  \\
Ave     & 0.092                       & 6.553                 & 0.000                                  & 0.399                       & 2.179                 & 0.079                                  & 0.241                       & 3.129                 & 0.468                                  & 0.229                       & 0.937                 & 0.344                                  \\ \bottomrule
\\
\end{tabular}

\label{table:arma}
\end{table*}
In this section, we conduct a series of experiments to evaluate the performance of the \re{Auto Regressive Moving Average (ARMA)}{ARMA} model in the context of predicting wrong-way cycling occurrences. Figure \ref{fig:ARMA} illustrates a comparative analysis between the ARMA(1,0) and ARMA(1,1) configurations. The notation ARMA($p$, $q$) denotes an ARMA model \re{}{where $p$ and $q$ are the orders of the autoregressive (AR) process and the moving average (MA) process, respectively.}
% where $p$ is the order of the auto regressive (AR) process, and $q$ is the order of the moving average (MA) process.

The analysis, particularly evident in \re{the first and third panels of subfigure (b)}{Case 2}, suggests that the inclusion of the MA component, which originally intended to capture traffic flow dynamics, does not contribute positively to the predictive accuracy for wrong-way cycling events. This is attributed to the observation that wrong-way cycling incidents rarely exhibit the traffic flow characteristics modeled by the MA component.

 Table \ref{table:arma} presents the respective standard errors, z-values, and p-values for auto regressive (AR) and moving average (MA) terms from ARMA models applied to two different cycling behaviors: Right-Way Cycling and Wrong-way Cycling. In the ARMA model analysis for cycling behaviors, \re{}{the wrong-way cycling data within ``Case 2"} exhibits an unusual case where the AR term appears completely non-significant with a p-value of 0.988, indicating a misinterpretation of the auto regressive effect as if it were a moving average component. Generally, for Wrong-way Cycling, the AR terms consistently show non-significance across all categories, suggesting a lack of moving average characteristic in this behavior. This pattern contrasts with the Right-way cycling data, where AR terms are significant, indicating a reliable auto regressive process.

To address this discrepancy, we incorporate domain knowledge into the model specification process. Consequently, we refine the ARMA model for wrong-way cycling prediction to an ARMA(1,0) configuration, effectively omitting the MA component. This adjustment is predicated on the rationale that the auto regressive component alone is more representative of the underlying process governing wrong-way cycling incidents, thereby enhancing the model's predictive relevance in this particular application.

\re{}{\section{Discussion}}
\re{}{\textbf{Limitations.} WWC-Predictor intentionally forgoes instance-level detection capabilities to optimize computational efficiency for its primary purpose of video-level analysis. While this trade-off enables effective system-wide monitoring of phenomena like wrong-way cycling ratios, it inherently sacrifices granularity for broader operational objectives. Consequently, the approach lacks resolution for fine-grained examination of individual objects or localized interactions, focusing instead on aggregate behavioral patterns across the surveillance network. }

\re{}{\textbf{Future work.} Future research will focus on enhancing cross-camera relationship modeling through graph-structured approaches to better capture traffic network topology. Additionally, significant efforts will address practical deployment challenges by optimizing the framework for edge devices. This includes developing lightweight architectures and efficient computation strategies suitable for real-world implementation in resource-constrained camera networks.}

\vspace{3mm}

\section{Conclusion}

In this paper, we introduced the new problem of wrong-way cycling ratio prediction in CCTV videos, and proposed a novel method, WWC-Predictor, to tackle this problem by sparse sampling with efficient Two-Frame WWC-Detector and Temporal WWC-Predictor, which has been mathematically and experimentally proven to be effective compared with straightforward tracking methods. Additionally, to facilitate the training and validation of our method for this task, we have presented and open-sourced three datasets to build a convincing benchmark of this task. \re{}{In our evaluation, our WWC Predictor demonstrates satisfactory performance with low computational resource demand.}

\bibliography{reference}

@INPROCEEDINGS{Yu_2023_CVPR,
  author={Yu, Yi and Da, Feipeng},
  booktitle={2023 IEEE/CVF Conference on Computer Vision and Pattern Recognition (CVPR)}, 
  title={Phase-Shifting Coder: Predicting Accurate Orientation in Oriented Object Detection}, 
  year={2023},
  volume={},
  number={},
  pages={13354-13363},
  doi={10.1109/CVPR52729.2023.01283}}

@article{instant-ngp,  
title={{Instant Neural Graphics Primitives with a Multiresolution Hash Encoding}}, 
url={http://dx.doi.org/10.1145/3528223.3530127}, 
DOI={10.1145/3528223.3530127}, 
journal={ACM Transactions on Graphics}, 
author={Müller, Thomas and Evans, Alex and Schied, Christoph and Keller, Alexander}, 
year={2022}, 
month={Jul}, 
pages={1–15}, 
language={en-US} 
}

@INPROCEEDINGS{schoenberger2016sfm,
  author={Schönberger, Johannes L. and Frahm, Jan-Michael},
  booktitle={2016 IEEE Conference on Computer Vision and Pattern Recognition (CVPR)}, 
  title={Structure-from-Motion Revisited}, 
  year={2016},
  volume={},
  number={},
  pages={4104-4113},
  doi={10.1109/CVPR.2016.445}}

@InProceedings{schoenberger2016mvs,
author="Sch{\"o}nberger, Johannes L.
and Zheng, Enliang
and Frahm, Jan-Michael
and Pollefeys, Marc",
editor="Leibe, Bastian
and Matas, Jiri
and Sebe, Nicu
and Welling, Max",
title="Pixelwise View Selection for Unstructured Multi-View Stereo",
booktitle="Computer Vision -- ECCV 2016",
year="2016",
publisher="Springer International Publishing",
pages="501--518",
isbn="978-3-319-46487-9"
}

@inproceedings{SORT,   title={{Simple Online and Realtime Tracking}},  url={http://dx.doi.org/10.1109/icip.2016.7533003},  DOI={10.1109/icip.2016.7533003},  booktitle={2016 IEEE International Conference on Image Processing (ICIP)},  author={Bewley, Alex and Ge, Zongyuan and Ott, Lionel and Ramos, Fabio and Upcroft, Ben},  year={2016},  month={Sep},  language={en-US}  }

@inproceedings{SORT_2,   title={{Simple Online and Realtime Tracking with a Deep Association Metric}},  url={http://dx.doi.org/10.1109/icip.2017.8296962},  DOI={10.1109/icip.2017.8296962},  booktitle={2017 IEEE International Conference on Image Processing (ICIP)},  author={Wojke, Nicolai and Bewley, Alex and Paulus, Dietrich},  year={2017},  month={Sep},  language={en-US}  }

@inproceedings{Xie_Cheng_Wang_Yao_Han_2021,   title={{Oriented R-CNN for Object Detection}},  url={http://dx.doi.org/10.1109/iccv48922.2021.00350},  DOI={10.1109/iccv48922.2021.00350},  booktitle={2021 IEEE/CVF International Conference on Computer Vision (ICCV)},  author={Xie, Xingxing and Cheng, Gong and Wang, Jiabao and Yao, Xiwen and Han, Junwei},  year={2021},  month={Oct},  language={en-US}  }

@inproceedings{Han_Ding_Xue_Xia_2021,   title={{ReDet: A Rotation-equivariant Detector for Aerial Object Detection}},  url={http://dx.doi.org/10.1109/cvpr46437.2021.00281},  DOI={10.1109/cvpr46437.2021.00281},  booktitle={2021 IEEE/CVF Conference on Computer Vision and Pattern Recognition (CVPR)},  author={Han, Jiaming and Ding, Jian and Xue, Nan and Xia, Gui-Song},  year={2021},  month={Jun},  language={en-US}  }

@software{YOLOV5,
author = {Jocher, Glenn},
doi = {10.5281/zenodo.3908559},
license = {AGPL-3.0},
month = may,
title = {{YOLOv5 by Ultralytics}},
url = {https://github.com/ultralytics/yolov5},
version = {7.0},
year = {2020}
}

@inproceedings{resnet,  
 title={{Deep Residual Learning for Image Recognition}}, 
 url={http://dx.doi.org/10.1109/cvpr.2016.90}, 
 DOI={10.1109/cvpr.2016.90}, 
 booktitle={2016 IEEE Conference on Computer Vision and Pattern Recognition (CVPR)}, 
 author={He, Kaiming and Zhang, Xiangyu and Ren, Shaoqing and Sun, Jian}, 
 year={2016}, 
 month={Jun}, 
 language={en-US} 
 }

@article{suttiponpisarn2022autonomous,
  title={{An Autonomous Framework For Real-time Wrong-way Driving Vehicle Detection from Closed-circuit Televisions}},
  author={Suttiponpisarn, Pintusorn and Charnsripinyo, Chalermpol and Usanavasin, Sasiporn and Nakahara, Hiro},
  journal={Sustainability},
  volume={14},
  number={16},
  pages={10232},
  year={2022},
  publisher={MDPI},
  doi={10.3390/su141610232}
}

@INPROCEEDINGS{choudhari2023traffic,
  author={Choudhari, Rutvik and Goel, Shubham and Patel, Yash and Ghane, Sunil},
  booktitle={2023 World Conference on Communication \& Computing (WCONF)}, 
  title={Traffic Rule Violation Detection using Detectron2 and Yolov7}, 
  year={2023},
  volume={},
  number={},
  pages={1-7},
  doi={10.1109/WCONF58270.2023.10235130}}

@software{yukai_yang_2020_4294717,
 author       = {Yukai Yang},
 title={{{FastMOT: High-performance Multiple Object Tracking Based on Deep SORT and KLT}}},
 month        = nov,
 year         = 2020,
 publisher    = {Zenodo},
 version      = {v1.0.0},
 doi          = {10.5281/zenodo.4294717},
 url          = {https://doi.org/10.5281/zenodo.4294717}
}

@INPROCEEDINGS{shubho2021real,
  author={Shubho, Fahimul Hoque and Iftekhar, Fahim and Hossain, Ekhfa and Siddique, Shahnewaz},
  booktitle={2021 IEEE Region 10 Conference (TENCON)}, 
  title={{Real-time traffic monitoring and traffic offense detection using YOLOv4 and OpenCV DNN}}, 
  year={2021},
  volume={},
  number={},
  pages={46-51},
  doi={10.1109/TENCON54134.2021.9707406}}

@INPROCEEDINGS{10421458,
  author={Vardhan, M. Harsha and Krishna, K. Venkata Sai and Munappa, Sunitha and Manoj, K. Aditya},
  booktitle={2023 International Conference on Next Generation Electronics (NEleX)}, 
  title={{Wrong Route Vehicles Detection Using Deep Learning}}, 
  year={2023},
  volume={},
  number={},
  pages={1-6},
  doi={10.1109/NEleX59773.2023.10421458}}

@INPROCEEDINGS{9230463,
  author={Rahman, Zillur and Ami, Amit Mazumder and Ullah, Muhammad Ahsan},
  booktitle={2020 IEEE Region 10 Symposium (TENSYMP)}, 
  title={{A Real-time Wrong-way Vehicle Detection Based on YOLO and Centroid Tracking}}, 
  year={2020},
  volume={},
  number={},
  pages={916-920},
  doi={10.1109/TENSYMP50017.2020.9230463}}

@article{tian2021emd,
  title={{An EMD and ARMA-based Network Traffic Prediction Approach in SDN-based Internet of Vehicles}},
  author={Tian, Miao and Sun, Chen and Wu, Shaozhi},
  journal={Wireless Networks},
  pages={1--13},
  year={2021},
  publisher={Springer},
  doi={10.1007/s11276-021-02675-2}
}

@inproceedings{peng2021short,
  title={{Short-term Traffic Flow Forecast Based on ARIMA-SVM Combined Model}},
  author={Peng, Jiaxin and Xu, Yongneng and Wu, Menghui},
  booktitle={2021 International Conference on Green Intelligent Transportation System and Safety},
  pages={287--300},
  year={2021},
  organization={Springer},
  doi={10.1007/978-981-19-5615-7_20}
}

@article{WEN2023119960,
title = {{A comprehensive survey of oriented object detection in remote sensing images}},
journal = {Expert Systems with Applications},
volume = {224},
pages = {119960},
year = {2023},
issn = {0957-4174},
doi = {https://doi.org/10.1016/j.eswa.2023.119960},
url = {https://www.sciencedirect.com/science/article/pii/S0957417423004621},
author = {Long Wen and Yu Cheng and Yi Fang and Xinyu Li},
}

@inproceedings{gu2017bikemate,
author = {Gu, Weixi and Zhou, Zimu and Zhou, Yuxun and Zou, Han and Liu, Yunxin and Spanos, Costas J. and Zhang, Lin},
title = {{BikeMate: Bike Riding Behavior Monitoring with Smartphones}},
year = {2017},
isbn = {9781450353687},
publisher = {Association for Computing Machinery},
url = {https://doi.org/10.1145/3144457.3144462},
doi = {10.1145/3144457.3144462},
pages = {313–322},
numpages = {10},
location = {Melbourne, VIC, Australia},
series = {MobiQuitous 2017}
}

@inproceedings{hayashi2021vision,
author = {Hayashi, Hirotaka and Xu, Anran and Zhou, Zhongyi and Yatani, Koji},
title = {{Vision-based Scene Analysis toward Dangerous Cycling Behavior Detection Using Smartphones}},
year = {2021},
isbn = {9781450384612},
publisher = {Association for Computing Machinery},
url = {https://doi.org/10.1145/3460418.3479300},
doi = {10.1145/3460418.3479300},
booktitle = {2021 Adjunct Proceedings of the 2021 ACM International Joint Conference on Pervasive and Ubiquitous Computing and Proceedings of the 2021 ACM International Symposium on Wearable Computers (UbiComp/ISWC)},
pages = {28–29},
numpages = {2},
location = {Virtual, USA},
series = {UbiComp/ISWC '21 Adjunct}
}

@article{dhakal2018using,
title = {{Using CyclePhilly data to assess wrong-way riding of cyclists in Philadelphia}},
journal = {Journal of Safety Research},
volume = {67},
pages = {145-153},
year = {2018},
issn = {0022-4375},
doi = {https://doi.org/10.1016/j.jsr.2018.10.004},
url = {https://www.sciencedirect.com/science/article/pii/S0022437517307338},
author = {Nirbesh Dhakal and Christopher R. Cherry and Ziwen Ling and Mojdeh Azad},
}

@INPROCEEDINGS{9648579,
  author={Suttiponpisarn, Pintusorn and Charnsripinyo, Chalermpol and Usanavasin, Sasiporn and Nakahara, Hiro},
  booktitle={2021 International Conference on Knowledge and Systems Engineering (KSE)}, 
  title={{Detection of Wrong Direction Vehicles on Two-way Traffic}}, 
  year={2021},
  volume={},
  number={},
  pages={1-6},
  doi={10.1109/KSE53942.2021.9648579}}

@article{bochkovskiy2020yolov4,
  title={{YOLOv4: Optimal Speed And Accuracy of Object Detection}},
  author={Bochkovskiy, Alexey and Wang, Chien-Yao and Liao, Hong-Yuan Mark},
  journal={arXiv preprint arXiv:2004.10934},
  year={2020}
}

@article{suttiponpisarn2022enhanced,
title = {{An Enhanced System for Wrong-Way Driving Vehicle Detection with Road Boundary Detection Algorithm}},
journal = {Procedia Computer Science},
volume = {204},
pages = {164-171},
year = {2022},
note = {2022 International Conference on Industry Sciences and Computer Science Innovation (iSCSi)},
issn = {1877-0509},
doi = {https://doi.org/10.1016/j.procs.2022.08.020},
url = {https://www.sciencedirect.com/science/article/pii/S187705092200758X},
author = {Pintusorn Suttiponpisarn and Chalermpol Charnsripinyo and Sasiporn Usanavasin and Hiro Nakahara},
}

@inproceedings{manasa2023enhanced,
  title={{An Enhanced Real-time System for Wrong-Way and Over Speed Violation Detection Using Deep Learning}},
  author={Manasa, A and Renuka Devi, SM},
booktitle="2023 International Conference on Image Processing and Capsule Networks (ICIPCN)",
year="2023",
publisher="Springer Nature Singapore",
pages="309--322",
isbn="978-981-99-7093-3",
doi="10.1007/978-981-99-7093-3_21"
}

@article{saho2017kalman,
author = {Kenshi Saho},
title = {{Kalman Filter for Moving Object Tracking: Performance Analysis and Filter Design}},
booktitle = {Kalman Filters},
publisher = {IntechOpen},
year = {2017},
editor = {Ginalber Luiz de Oliveira Serra},
chapter = {12},
doi = {10.5772/intechopen.71731},
url = {https://doi.org/10.5772/intechopen.71731}
}

@inproceedings{zhang2022bytetrack,
author = {Zhang, Yifu and Sun, Peize and Jiang, Yi and Yu, Dongdong and Weng, Fucheng and Yuan, Zehuan and Luo, Ping and Liu, Wenyu and Wang, Xinggang},
title = {{ByteTrack: Multi-object Tracking by Associating Every Detection Box}},
year = {2022},
isbn = {978-3-031-20046-5},
publisher = {Springer-Verlag},
url = {https://doi.org/10.1007/978-3-031-20047-2_1},
doi = {10.1007/978-3-031-20047-2_1},
booktitle = {2022 European Conferenceon Computer Vision (ECCV)},
pages = {1–21},
numpages = {21},
location = {Tel Aviv, Israel}
}

@software{josef_perktold_2023_10378921,
  author       = {Josef Perktold and
                  Skipper Seabold and
                  Kevin Sheppard and
                  ChadFulton and
                  Kerby Shedden and
                  jbrockmendel and
                  j-grana6 and
                  Peter Quackenbush and
                  Vincent Arel-Bundock and
                  Wes McKinney and
                  Ian Langmore and
                  Bart Baker and
                  Ralf Gommers and
                  yogabonito and
                  s-scherrer and
                  Yauhen Zhurko and
                  Matthew Brett and
                  Enrico Giampieri and
                  yl565 and
                  Jarrod Millman and
                  Paul Hobson and
                  Vincent and
                  Pamphile Roy and
                  Tom Augspurger and
                  tvanzyl and
                  alexbrc and
                  Tyler Hartley and
                  Fernando Perez and
                  Yuji Tamiya and
                  Yaroslav Halchenko},
  title={{statsmodels/statsmodels: Release 0.14.1}},
  month        = dec,
  year         = 2023,
  publisher    = {Zenodo},
  version      = {v0.14.1},
  doi          = {10.5281/zenodo.10378921},
  url          = {https://doi.org/10.5281/zenodo.10378921}
}

@article{mcleod2008faster,
author = {McLeod, A. I. and Zhang, Y.},
title = {{Faster ARMA maximum likelihood estimation}},
year = {2008},
issue_date = {January, 2008},
publisher = {Elsevier Science Publishers B. V.},
volume = {52},
number = {4},
issn = {0167-9473},
url = {https://doi.org/10.1016/j.csda.2007.07.020},
doi = {10.1016/j.csda.2007.07.020},
journal = {Computational Statistics and Data Analysis},
month = jan,
pages = {2166–2176},
numpages = {11}
}

@ARTICLE{CCTV-video,
  author={Buch, Norbert and Velastin, Sergio A. and Orwell, James},
  journal={IEEE Transactions on Intelligent Transportation Systems}, 
  title={{A Review of Computer Vision Techniques for the Analysis of Urban Traffic}}, 
  year={2011},
  volume={12},
  number={3},
  pages={920-939},
  doi={10.1109/TITS.2011.2119372}}

@article{Isufi_2017,
   title={{Autoregressive Moving Average Graph Filtering}},
   volume={65},
   ISSN={1941-0476},
   url={http://dx.doi.org/10.1109/TSP.2016.2614793},
   DOI={10.1109/tsp.2016.2614793},
   number={2},
   journal={IEEE Transactions on Signal Processing},
   publisher={Institute of Electrical and Electronics Engineers (IEEE)},
   author={Isufi, Elvin and Loukas, Andreas and Simonetto, Andrea and Leus, Geert},
   year={2017},
   month=jan, pages={274–288} }

@ARTICLE{Shuman2013,
  author={Shuman, David I and Narang, Sunil K. and Frossard, Pascal and Ortega, Antonio and Vandergheynst, Pierre},
  journal={IEEE Signal Processing Magazine}, 
  title={The emerging field of signal processing on graphs: Extending high-dimensional data analysis to networks and other irregular domains}, 
  year={2013},
  volume={30},
  number={3},
  pages={83-98},
  doi={10.1109/MSP.2012.2235192}}

@InProceedings{Giraldo2021Moving,
author="Giraldo, Jhony H.
and Javed, Sajid
and Sultana, Maryam
and Jung, Soon Ki
and Bouwmans, Thierry",
editor="Jeong, Hieyong
and Sumi, Kazuhiko",
title="The Emerging Field of Graph Signal Processing for Moving Object Segmentation",
booktitle="Frontiers of Computer Vision",
year="2021",
publisher="Springer International Publishing",
pages="31--45",
isbn="978-3-030-81638-4",
doi="10.1007/978-3-030-81638-4_3"
}

@ARTICLE{Leus2023,
  author={Leus, Geert and Marques, Antonio G. and Moura, José M.F. and Ortega, Antonio and Shuman, David I},
  journal={IEEE Signal Processing Magazine}, 
  title={Graph Signal Processing: History, development, impact, and outlook}, 
  year={2023},
  volume={40},
  number={4},
  pages={49-60},
  doi={10.1109/MSP.2023.3262906}}

@inproceedings{
dosovitskiy2021an,
title={An Image is Worth 16x16 Words: Transformers for Image Recognition at Scale},
author={Alexey Dosovitskiy and Lucas Beyer and Alexander Kolesnikov and Dirk Weissenborn and Xiaohua Zhai and Thomas Unterthiner and Mostafa Dehghani and Matthias Minderer and Georg Heigold and Sylvain Gelly and Jakob Uszkoreit and Neil Houlsby},
booktitle={International Conference on Learning Representations},
year={2021},
url={https://openreview.net/forum?id=YicbFdNTTy}
}

@INPROCEEDINGS{10725591,
  author={Gaur, Krish and Siddique, Miran Ahmad and Beernally, Krishna and Madaan, Nityam and Tarwani, Sandhya},
  booktitle={2024 International Conference on Computing Communication and Networking Technologies (ICCCNT)}, 
  title={Real-Time Wrong-Way Vehicle Detection System with Automatic Number Plate Recognition for Enhanced Road Safety}, 
  year={2024},
  volume={},
  number={},
  pages={1-8},
  doi={10.1109/ICCCNT61001.2024.10725591}}
\bibliographystyle{IEEEtran}

\end{document}